\documentclass[letterpaper,10pt,conference]{ieeeconf}  

\IEEEoverridecommandlockouts                             
\PassOptionsToPackage{export}{adjustbox}
\usepackage{definitions}
\usepackage[T1]{fontenc}
\usepackage{tikz}
\usepackage{flushend}

\usepackage{dsfont}
\usepackage[cal=pxtx,scr=boondox]{mathalpha}
\usepackage{bm}
\usepackage{yhmath}
\usepackage{adjustbox}
\usepackage{mathtools}
\usepackage{amsmath}
\usepackage{amssymb} 
\usepackage[linesnumbered,ruled,vlined]{algorithm2e}
\usepackage{url}
\usepackage{graphics}
\usepackage{graphicx}
\usepackage{epsfig}
\usepackage{textcomp}
\usepackage{cite}
\usepackage{bm}
\usepackage{subcaption}
\usepackage[labelfont=bf]{caption}
\usepackage{booktabs}
\usepackage[english]{babel}
\usepackage[utf8]{inputenc}
\usepackage{theorem}
\usepackage{pifont}
\usepackage{cuted}
\usepackage[dvipsnames]{colortbl,xcolor}
\usepackage{multirow}
\usepackage{siunitx}
\usepackage{gensymb}
\usepackage{stmaryrd}
\usepackage{makecell}
\usepackage{float}
\let\labelindent\relax
\usepackage{enumitem}

\definecolor{kitred}{RGB}{187,25,23}

\newtheorem{proposition}{Proposition}
\usepackage{hyperref}
\usepackage{url}

\title{\LARGE \bf Open-DiffLoco: Open-Source Differentiable Learning for Deployable Blind Quadruped Locomotion}

\author{Martin Opat \\
University of Groningen, Netherlands \\
{\tt\small martinopat.opat@gmail.com}
}

\begin{document}

\maketitle
\thispagestyle{empty}
\pagestyle{empty}

\begin{abstract}
    Developing deployable locomotion policies through conventional reinforcement learning often requires complex reward engineering and expensive training times. While differentiable simulation offers a highly efficient alternative, open-source tools capable of end-to-end transfer of these policies to physical hardware remain limited.
	This paper introduces Open-DiffLoco, an open-source framework for training deployable blind quadruped locomotion policies with differentiable simulation. The framework implements the Short-Horizon Actor-Critic (SHAC) algorithm in MuJoCo XLA (MJX) and trains a proprioceptive policy that transfers to real-world hardware. The deployed policy removes privileged actor observations, including base linear velocity, and does not rely on reference trajectories. It also uses a substantially simplified reward function, enabling the robot to discover walking patterns without the complex auxiliary rewards typically used in conventional reinforcement learning pipelines. When deployed on physical hardware (a Unitree Go2 quadruped), the trained policy tracks omnidirectional velocity commands with root-mean-square error below \SI{0.2}{\meter/\second}, reaches speeds above \SI{1}{\meter/\second}, and remains robust to uneven terrain and external physical disturbances, such as lateral pushes. Across the reported configurations, training uses under 6 GB of VRAM on a single NVIDIA GeForce RTX 5080 GPU and completes in approximately 20--60 minutes. As an algorithmic extension to SHAC, we propose Jacobian-Augmented Value Estimation (JAVE), which supervises the critic Jacobians to improve early first-order policy-gradient training. To our knowledge, Open-DiffLoco is the first open-source framework for training deployable locomotion policies using differentiable simulation. Deployment videos and source code are available at: \href{https://diffloco.martin-opat.com/}{https://diffloco.martin-opat.com/}
\end{abstract}


\section{Introduction}\label{sec:intro}
The interest in embodied artificial intelligence (AI), a physical body controlled by an AI algorithm, is intensified by humankind's passion for automation and technology, but weakened by its fear of a dystopian future~\cite{negativeAttitudesSyrdal2009, RobotHumanGoding2023}.
Nevertheless, embodied AI is an increasingly relevant topic in modern machine learning research.
Recently, legged robots have emerged as versatile platforms capable of traversing challenging environments.
Unlike wheeled embodiments, which require mostly flat terrain, legged robots are suited for stepping over gaps and obstacles, and traversing steep inclines and stairs.
This adaptability makes legged systems heavily researched across various sectors.
Their applications span from agriculture~\cite{agriEx} and search-and-rescue operations in disaster zones~\cite{disaster} to future exploration of extraterrestrial objects~\cite{spaceEx} and many more.

In reinforcement learning (RL), the mapping from observations to actions is called a policy and is typically represented by a neural network. The goal of RL-based locomotion training is to discover an effective policy for a given robotic task. While model-based RL approaches exist~\cite{modelRLsurveyMoerland2022}, modern robot learning predominantly uses model-free RL algorithms, most notably Proximal Policy Optimisation (PPO)~\cite{ppo} implemented in frameworks such as RSL-RL~\cite{schwarke2025rslrl}. Model-free RL treats the simulation environment as a black box and learns through trial-and-error interactions. Lacking access to the underlying physical equations, it assumes a stochastic state-transition model and relies on zeroth-order gradients (ZoG)~\cite{bpttMujocoLuo2024}. Although this model-free paradigm is robust and flexible, it suffers from low sample efficiency due to weakly informative ZoG estimates, leading to prolonged training times and heavy computational resource requirements.

Differentiable simulation addresses these limitations by treating the simulator as a glass box, computing highly informative first-order gradients (FoG) by backpropagating through the environment's dynamics.
While this approach significantly improves sample efficiency, it introduces new constraints: dynamics and reward functions must be differentiable, and noisy gradient calculations frequently cause numerical instability~\cite{metz2022gradientsneed, suh2022differentiablesimulatorsbetterpolicy}.
Consequently, existing literature primarily demonstrates low-speed\footnote{Less than or equal to \SI{1}{\meter\per\second}~\cite{bpttMujocoLuo2024, schwarke2025, deploy2Song2025}.} locomotion policies. Although short-horizon policy learning, most notably the Short-Horizon Actor-Critic (SHAC) algorithm~\cite{shacXu2022}, helps mitigate this instability, broader community adoption remains hindered by the current state of the art relying on closed-source and custom-built frameworks~\cite{schwarke2025, deploy2Song2025, bpttMujocoLuo2024}. To address these issues, this work develops a method for training deployable policies within a mainstream, accessible simulation ecosystem -- MuJoCo XLA (MJX)~\cite{todorov2012mujoco}.

\subsection*{Contributions}
Our end-to-end differentiable framework leverages first-order gradients to train robust, deployable quadrupedal locomotion policies. The core contributions are as follows:
\begin{itemize}[leftmargin=*]
	\item Open-Source Differentiable Framework: We introduce the first open-source framework for training deployable locomotion policies using differentiable simulation, with the Short-Horizon Actor-Critic (SHAC) algorithm implemented in MuJoCo XLA.
	\item SHAC Algorithm Extension: We propose and evaluate Jacobian-Augmented Value Estimation (JAVE), an extension to SHAC that augments critic learning with value-Jacobian supervision to improve early first-order policy-gradient training.
	\item Simplified Reward Structure and Motion Autonomy: We demonstrate that analytical gradients can eliminate the need for reference trajectories and extensive reward engineering; our framework discovers stable walking patterns while removing the complex auxiliary rewards, such as gait pattern rewards, that traditional algorithms like PPO often require for reliable convergence.
	\item Minimalist Observation Space: We remove privileged information, such as base linear velocity, from the actor observations, yielding a robust policy that relies entirely on onboard proprioception.
	\item Real-World Validation: We provide a hardware evaluation demonstrating the efficacy of differentiable simulation for legged robots under omnidirectional velocity tracking, uneven terrain, and external physical disturbances, demonstrated in simulation and on the Unitree Go2 quadruped.
\end{itemize}

\section{Framework Design} \label{sec:framework-design}
\subsection{Framework Modularity}
\begin{figure*}[t]
    \centering
    \includegraphics[width=\linewidth]{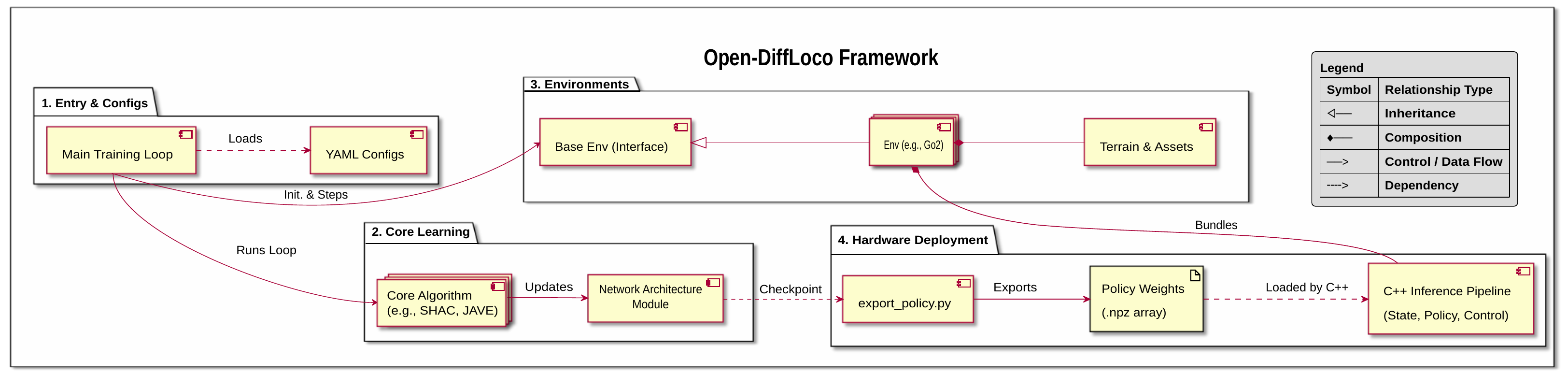}
    \caption{High-level system architecture of the Open-DiffLoco framework.}
    \label{fig:arch}
\end{figure*}
The proposed framework, illustrated in \autoref{fig:arch}, is designed for modularity. As depicted, the system loads a YAML configuration file, which acts as an entry point, allowing users to adapt the framework to different robotic platforms. This configuration file defines embodiment-specific components, such as the reward function and physical structure.
Within the core learning pipeline, both the training algorithm and the network architecture are interchangeable modules. For the algorithm, users can choose between a general SHAC implementation based on the open-source DiffRL framework~\cite{shacXu2022} and alternatives like JAVE (see Section~\ref{subsec: JAVE}). Furthermore, all networks are currently implemented as simple multilayer perceptrons (MLPs), but the modular design allows easy changes to the network architecture. 
Finally, a policy can be deployed to real hardware using a custom, embodiment-specific C++ and ROS2~\cite{ros2} stack.

\subsection{Reward Functions}
Although traditional RL uses gradient-based optimisation, it does not require strict differentiability because the stochastic nature of the algorithms statistically smooths discontinuities and kinks across thousands of parallel environments. In contrast, FoG methods lack this statistical smoothing because they rely on analytical gradients and use over an order of magnitude fewer parallel environments. This absence makes strictly differentiable rewards necessary, posing a notable challenge because many robust and validated rewards used in traditional model-free setups rely on instantaneous peaks or discontinuous formulations. Most notably, rewards that define gait patterns, such as airtime rewards or gait-specific rewards that encourage certain limbs to be airborne simultaneously, frequently rely on non-differentiable step functions.

Current works address this limitation by proposing differentiable equivalents of these rewards. However, this approach often comes with significant downsides, such as introducing additional hyperparameters that can ultimately bottleneck policy performance. A prime example of this is the trotting gait rewards proposed by Schwarke et al.~\cite{schwarke2025} or the kinematic reference used by Luo et al.~\cite{bpttMujocoLuo2024} in their forward locomotion example. Both formulations use a fixed frequency parameter that inherently dictates how often legs should lift, limiting the robot's range of motion and potentially preventing the discovery of high-speed emergent behaviours.

We found that most such rewards are unnecessary, especially for straightforward tasks like blind locomotion. In this work, blind locomotion refers to locomotion without exteroceptive terrain sensing, where the policy must rely on proprioceptive observations available during deployment. By using a minimal subset of differentiable rewards from PPO's standard reward set, policies trained in our framework discover stable gait patterns that yield robust, deployable policies. Crucially, we were unable to replicate these results when using PPO's implementation in MuJoCo Playground~\cite{mujoco_playground_2025} under the same setup.

As explored in~\cite{suh2022differentiablesimulatorsbetterpolicy}, the ZoG estimates used by PPO have high variance but remain unconditionally unbiased, in contrast to the analytical gradients of differentiable simulation, which have low variance but become biased in the presence of discontinuities, such as contacts. This insight explains why PPO struggles to converge to a satisfactory solution under a minimal reward set: the auxiliary shaping rewards that are typically present serve largely to densify the learning signal and suppress the variance of the estimator. Without them, the gradient estimate is too noisy to reliably discover a stable gait. Differentiable simulation, on the other hand, obtains a dense temporal signal directly by backpropagating through the dynamics, so the same information that shapes rewards in PPO is supplied ``for free'' by the analytical gradients. The fact that SHAC succeeds in this setting demonstrates that the blind-locomotion objective is sufficiently smooth for the low bias introduced by contact discontinuities to be outweighed by the substantial reduction in gradient variance. This optimisation advantage allows differentiable simulation to thrive in this simple setting.

The exact reward formulation described in this section is presented in 
\autoref{tab:go2_nolinvel_nokinref_rewards}.
Note that the quadratic tracking rewards can be replaced with a Gaussian formulation, which may be preferable for more advanced tasks.

\subsection{Training Cost}
The computational time complexity for policy training is $\mathcal O(I \cdot N \cdot H)$, where $I$ is the total number of training iterations, $N$ is the number of parallel environments, and $H$ is the unroll horizon length. The (VRAM) memory complexity scales as $\mathcal O(N \cdot H)$, since differentiable simulation requires storing the full computational graph for each environment with all the intermediate states to enable backpropagation.
All presented training results were obtained using a single NVIDIA GeForce RTX 5080 GPU. The entire training process takes as little as $\sim$\SI{20}{\min} and consumes under 6~GB of VRAM out of the 16~GB available, demonstrating that high-performance differentiable policy training is accessible without extensive compute clusters.

\begin{table}[t]
	\centering
	\caption{Simplistic reward terms used to train omnidirectional blind locomotion.}
	\label{tab:go2_nolinvel_nokinref_rewards}
	\begin{tabular}{lll}
		\hline
		\textbf{Term} & \textbf{Expression}                       \\
		\hline
		$x$ velocity tracking
		              & $-(v_x - v_x^{\mathrm{cmd}})^2$
		\\

		$y$ velocity tracking
		              & $-(v_y - v_y^{\mathrm{cmd}})^2$
		\\

		Yaw-rate tracking
		              & $-(\omega_z - \omega_z^{\mathrm{cmd}})^2$
		\\

		Height
		              & $\exp\!\left(-10(h - 0.3)^2\right)$
		\\

		Vertical velocity
		              & $-0.5 v_z^2$
		\\

		Upright
		              & $0.5(-g_z)$
		\\

		Joint deviation
		              & $-0.3 \sum_i (q_i - q_i^{0})^2$
		\\

		Action rate
		              & $-0.02 \sum_i (a_i - a_{i,t-1})^2$
		\\

		Action magnitude
		              & $-0.05 \sum_i a_i^2$
		\\

		Roll/pitch angular velocity
		              & $-0.05(\omega_x^2 + \omega_y^2)$
		\\
		\hline
	\end{tabular}
\end{table}

\subsection{Jacobian-Augmented Value Estimation} \label{subsec: JAVE}
As an algorithmic extension to the baseline framework, we propose Jacobian-Augmented Value Estimation (JAVE).
First-order policy-gradient methods compute the actor update by differentiating through an unrolled simulator trajectory. For a dynamics function $\mathcal F$ and policy $\pi_\theta$, the action $\mathbf{a}_t$ and state $\mathbf{s}_t$ adjoints satisfy
\begin{align}
	\frac{\partial \mathcal L}{\partial a^j_t}
	 & = \frac{\partial l}{\partial a^j_t}
	+ \frac{\partial \mathcal L}{\partial s^i_{t+1}}
	\frac{\partial \mathcal F^i}{\partial a^j_t}\,, \label{eq:app-action-adjoint} \\
	\frac{\partial \mathcal L}{\partial s^i_t}
	 & = \frac{\partial l}{\partial s^i_t}
	+ \frac{\partial \mathcal L}{\partial s^l_{t+1}}
	\frac{\partial \mathcal F^l}{\partial s^i_t}
	+ \frac{\partial \mathcal L}{\partial a^j_t}
	\frac{\partial \pi^j}{\partial s^i_t}\,, \label{eq:app-state-adjoint}
\end{align}
with terminal condition $\partial \mathcal L / \partial \mathbf s_{H+1}=0$ at the trajectory of length $H$, and implicit summation over repeated indices. The resulting actor gradient is
\begin{equation}
	(\nabla_\theta \mathcal L)_k =
	\sum_{t=0}^{H}
	\frac{\partial \mathcal L}{\partial a^j_t}
	\frac{\partial \pi^j}{\partial \theta^k}\,.
	\label{eq:app-bptt-loss-grad}
\end{equation}
Long-horizon backpropagation through contact-rich dynamics is numerically unstable because products of dynamics Jacobians can amplify perturbations. SHAC~\cite{shacXu2022} reduces this issue by truncating the differentiable rollout across $N$ parallel environments to a short horizon $h$, starting from time step $t_0$. The remaining return is replaced by a learned value function $V_\phi$, discounted by a factor $\gamma$:
\begin{align}
	\mathcal L(\theta) =
	\frac{1}{Nh}\sum_{i=1}^{N}
	 & \left[
		\sum_{t=t_0}^{t_0+h-1} \gamma^{t-t_0}
		l(\mathbf s^{(i)}_t,\mathbf a^{(i)}_t)
	\right. \nonumber \\
	 & \left.
		- \gamma^h V_\phi(\mathbf s^{(i)}_{t_0+h})
		\right].
	\label{eq:app-shac-actor}
\end{align}
The critic is trained with a TD-$\lambda$ target $\hat V_t$,
\begin{equation}
	\mathcal L_{\mathrm{critic}}(\phi)
	= \mathbb E_t\left[
		\left\|V_\phi(\mathbf s_t)-\hat V_t\right\|^2
		\right].
	\label{eq:app-shac-critic}
\end{equation}
However, the actor gradient induced by \eqref{eq:app-shac-actor} does not only depend on the value estimate but also on its gradient at the truncation boundary,
\begin{equation}
	\frac{\partial \mathcal L}
	{\partial \mathbf s^{(i)}_{t_0+h}}
	= -\frac{\gamma^h}{Nh}
	\frac{\partial V_\phi}{\partial \mathbf s}
	\bigg|_{\mathbf s^{(i)}_{t_0+h}}\,,
	\label{eq:app-boundary-adjoint}
\end{equation}
which is then propagated backwards through every preceding step in the rollout. A pointwise-accurate critic can therefore still bias the policy update if its local Jacobian is inaccurate. See Appendix~\ref{appendix:grad-unbound-proof} for a full proof.

JAVE addresses this by directly supervising the value-function Jacobian. Let $\mathbf o^c\in\mathcal O_c$ denote the critic observation, $\mathbf o^a\in\mathcal O_a$ the actor observation, and let $\mathcal G_\beta$ denote the one-step observation map under the behaviour policy $\beta$ that generated the sampled transition:
\begin{equation}
	\mathbf o^c_{t+1}=\mathcal G_\beta(\mathbf o^c_t)\,.
\end{equation}
The corresponding value function satisfies the Bellman equation~\cite{BellmanEqSutton1998}
\begin{align}
	V(\mathbf o^c)
	 & = r(\mathbf o^c,\mathcal G_\beta(\mathbf o^c),\mathbf a)
	\nonumber                                                   \\
	 & \quad
	+ \gamma V(\mathcal G_\beta(\mathbf o^c))\,,
	\qquad \mathbf a\sim\beta(\cdot\mid\mathbf o^a)\,,
	\label{eq:app-bellman}
\end{align}
where $r$ is the reward function.
Differentiating \eqref{eq:app-bellman} gives the Bellman-gradient relation
\begin{align}
	(\nabla_{\mathbf o^c}V)_i
	 & = \frac{\partial r}{\partial (o^c)^i}
	+ \frac{\partial r}{\partial (o^{c\prime})^k}
	\frac{\partial \mathcal G_\beta^k}{\partial (o^c)^i}
	\nonumber                                \\
	 & \quad
	+ \gamma
	\frac{\partial \mathcal G_\beta^k}{\partial (o^c)^i}
	\left(
	\left.\nabla_{\mathbf o^c}V\right|_{\mathcal G_\beta(\mathbf o^c)}
	\right)_k\,.
	\label{eq:app-grad-bellman}
\end{align}
This formulation allows asymmetric actor-critic observations. The Bellman-gradient target is constructed in critic-observation space, while the sampled action is initially treated as fixed, meaning the actor may receive a smaller deployment-compatible observation.

The final JAVE critic objective augments the TD-$\lambda$ value loss $\mathcal L_{\mathrm{TD}}$ with gradient Bellman loss $\mathcal L_{\mathrm{GB}}$:
\begin{align}
	\mathcal L_{\mathrm{critic\text{-}JAVE}}(\phi)
	 & = \alpha_{\mathrm{TD}}\mathcal L_{\mathrm{TD}}(\phi)
	+ \alpha_{\mathrm{GB}}\mathcal L_{\mathrm{GB}}(\phi),
	\label{eq:app-jave-loss}
\end{align}
where $\alpha_{\mathrm{TD}}$ and $\alpha_{\mathrm{GB}}$ are scalar weights introduced to balance the two losses.
The two loss components are defined using the stop-gradient operator $\mathrm{sg}\{\cdot\}$
\begin{align}
	\mathcal L_{\mathrm{TD}}(\phi)
	 & = \mathbb E_t\left[
	\left\|V_\phi(\mathbf o^c_t)
	-\mathrm{sg}\{\hat V_t^{\mathrm{TD}(\lambda)}\}
	\right\|^2
	\right], \label{eq:app-jave-td} \\
	\mathcal L_{\mathrm{GB}}(\phi)
	 & = \mathbb E_t\left[
	\left\|
	\nabla_{\mathbf o^c}V_\phi\big|_{\mathbf o^c_t}
	-\mathrm{sg}\{\hat\zeta_t\}
	\right\|^2
	\right]. \label{eq:app-jave-gb}
\end{align}
The TD term fixes the value scale, while the Bellman-gradient term fixes the local geometry of the critic.

Since the observation function is non-invertible, the true map $\mathcal G_\beta$ is not available in closed form. JAVE therefore estimates the observation-space dynamics with a learned dynamics model $\mathcal F_\psi$. The model is purely geometric and predicts only the critic-observation residual,
\begin{equation}
	\mathcal F_\psi(\mathbf o^c,\mathbf a)
	\approx
	\mathbf o^{c\prime}-\mathbf o^c,
	\label{eq:app-jave-model}
\end{equation}
with the reward used in the target constructed analytically from the critic observations and action. For the sampled executed action $\mathbf a_t\sim\beta(\cdot\mid\mathbf o^a_t)$, define $\bar{\mathbf a}_t=\mathrm{sg}\{\mathbf a_t\}$ and
\begin{equation}
	\hat{\mathcal G}_\beta(\mathbf o^c_t,\bar{\mathbf a}_t)
	= \mathbf o^c_t + \mathcal F_\psi(\mathbf o^c_t,\bar{\mathbf a}_t).
	\label{eq:app-wm-closed-loop}
\end{equation}
Using a target critic $V_{\bar\phi}$, define
\begin{align}
	\hat{\mathbf o}^c_{t+1}
	 & = \hat{\mathcal G}_\beta(\mathbf o^c_t,\bar{\mathbf a}_t),    \\
	\hat r_t
	 & = r(\mathbf o^c_t,\hat{\mathbf o}^c_{t+1},\bar{\mathbf a}_t).
\end{align}
The JAVE gradient target then is
\begin{equation}
	(\hat\zeta_t)_i =
	\frac{\partial}{\partial (o^c)^i}
	\left[\hat r_t + \gamma V_{\bar\phi}(\hat{\mathbf o}^c_{t+1})\right]\,.
	\label{eq:app-gb-target-wm}
\end{equation}
Additional remarks and theoretical derivations are presented in Appendix~\ref{app:jave}.

\section{Evaluation in Simulation}\label{sec:ablation}
The policies are evaluated in MJX using a simulated Unitree Go2 quadruped. The control policy operates at \SI{50}{\hertz}, while the underlying physics engine is stepped at \SI{250}{\hertz}, executing 5 simulation substeps per control action. 
The complete parameterisation of the training environment -- including domain randomisation bounds, reward scales, and exact observation histories -- is available in the open-source repository \href{https://github.com/MartinOpat/open-diffloco}{https://github.com/MartinOpat/open-diffloco}.

We evaluate policies along two implementation axes: whether the actor receives a base linear velocity, which requires an onboard velocity estimator during deployment, and whether training uses a kinematic reference, inspired by the residual training setup of Luo et al.~\cite{bpttMujocoLuo2024}. The main deployed policy removes both to showcase flexibility. However, to isolate the optimiser's behaviour, the evaluated FoG methods retain the velocity estimator and kinematic reference parameterisations, as detailed in \autoref{sec:framework-design}. Consequently, in 
\autoref{fig:method-comparison}, PPO uses its native reward set and the FoG methods use the aforementioned reference-based parameterisations, so the figure should be read as a comparison of learning trends and sample efficiency rather than as an exact comparison of absolute reward scale.
PPO, SHAC, and JAVE are compared in 
\autoref{fig:method-comparison} in terms of both per-iteration and per-environment-step efficiency. The total number of environment steps is determined by $N_\text{iterations} \times N_\text{envs} \times h$, where $h$ corresponds to the unroll length for PPO and the horizon length for SHAC and JAVE.
The left plot shows that JAVE with a horizon length of $32$ and PPO perform significantly better than the other FoG methods with a horizon length of $16$.
However, the plot on the right shows that all three FoG methods perform substantially better than PPO in terms of reward gained per environment step\footnote{Also referred to as sample efficiency.}. Among the FoG methods, JAVE with a horizon length of $32$ performs best.
\begin{figure}[t]
	\vspace{1mm}
	\centering
	\includegraphics[width=\linewidth]{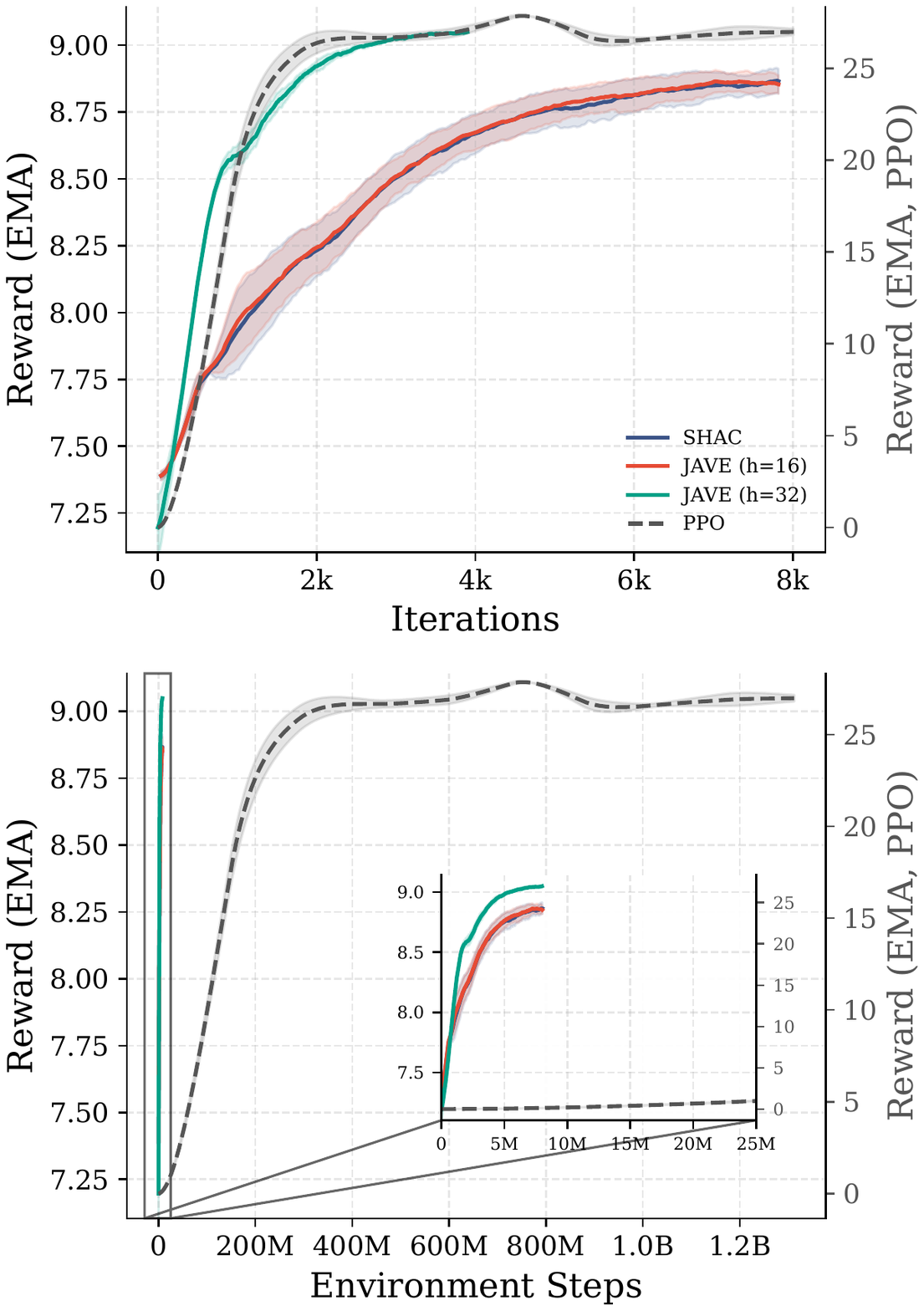}
	\captionof{figure}{Comparison of JAVE, SHAC, and PPO in terms of training exponential moving average (EMA) reward as a function of training iterations (left) and number of training steps across all environments (right). PPO uses its native reward implementation, so absolute reward values are not directly comparable across methods.}
	\label{fig:method-comparison}
\end{figure}

\begin{figure}
	\centering



    \includegraphics[width=\linewidth]{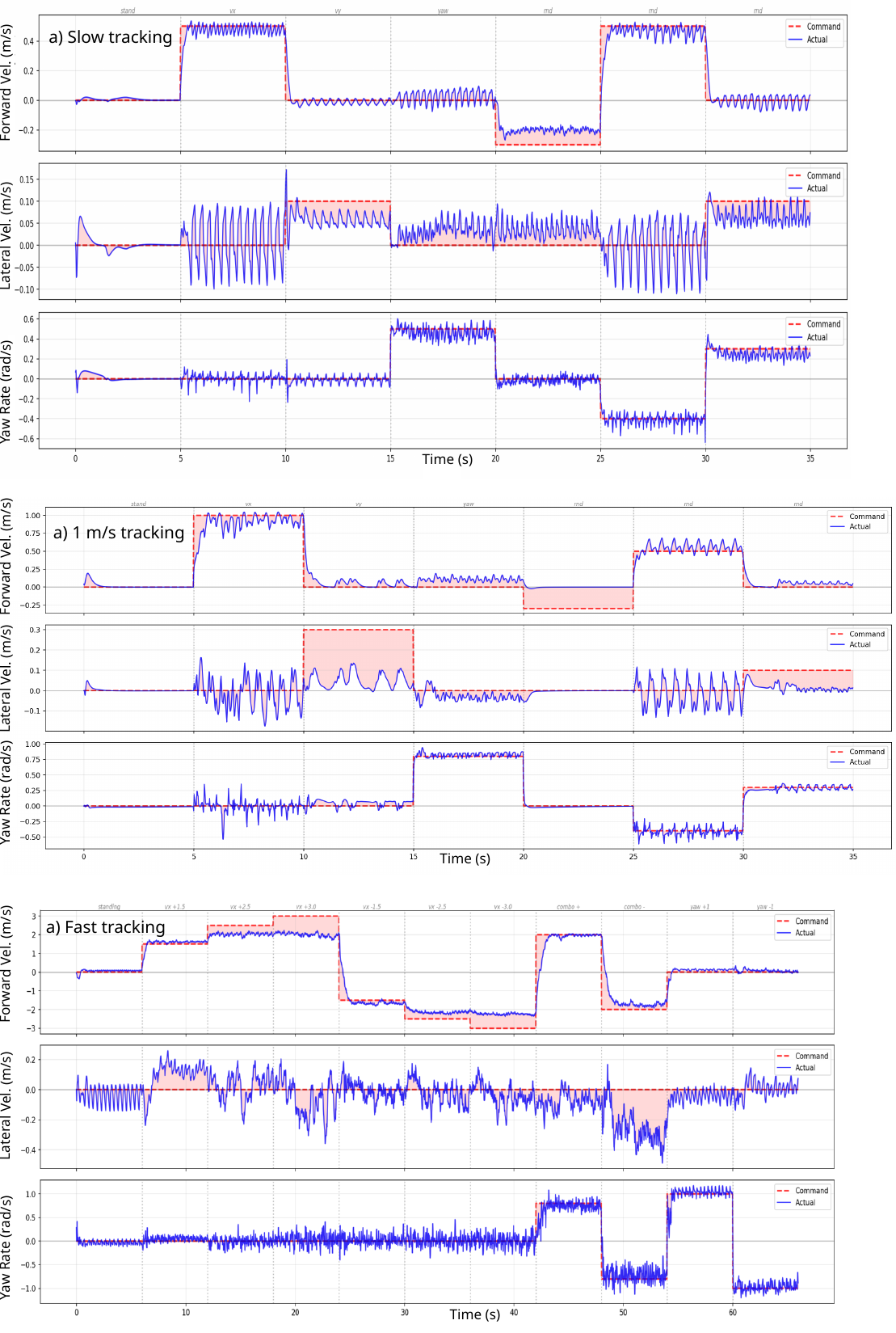}
	\caption{In-simulation velocity tracking evaluation across different speed profiles.}
	\label{fig:velocity_evaluation}
\end{figure}

\autoref{fig:velocity_evaluation} shows the in-simulation velocity-tracking accuracy across three speed profiles. Both the slow profile (a) and the \SI{1}{\meter/\second} profile (b) show excellent tracking accuracy for the forward-velocity and yaw-rate commands. Tracking accuracy for the lateral-velocity command, for which the trained speed range was the smallest, is traded for greater precision in the other two components. For the fast profile (c), the lateral command was set to zero to prioritise high-speed forward motion while maintaining adequate yaw-rate tracking. The plot shows a clear performance plateau because the maximum speed the policy can achieve is \SI{2}{\meter/\second}.

\section{Real-World Experiments}\label{sec:experiment}
The proposed framework provides custom C++ deployment code that runs fully onboard a Unitree Go2 EDU robot equipped with a Jetson AGX Orin (64 GB).
The deployment code relies only on the \texttt{unitree\_sdk2} package, with an optional ROS2 dependency for command input. The code can therefore be used independently or with ROS2.

\begin{figure}[t]
	\centering
	\includegraphics[width=\linewidth]{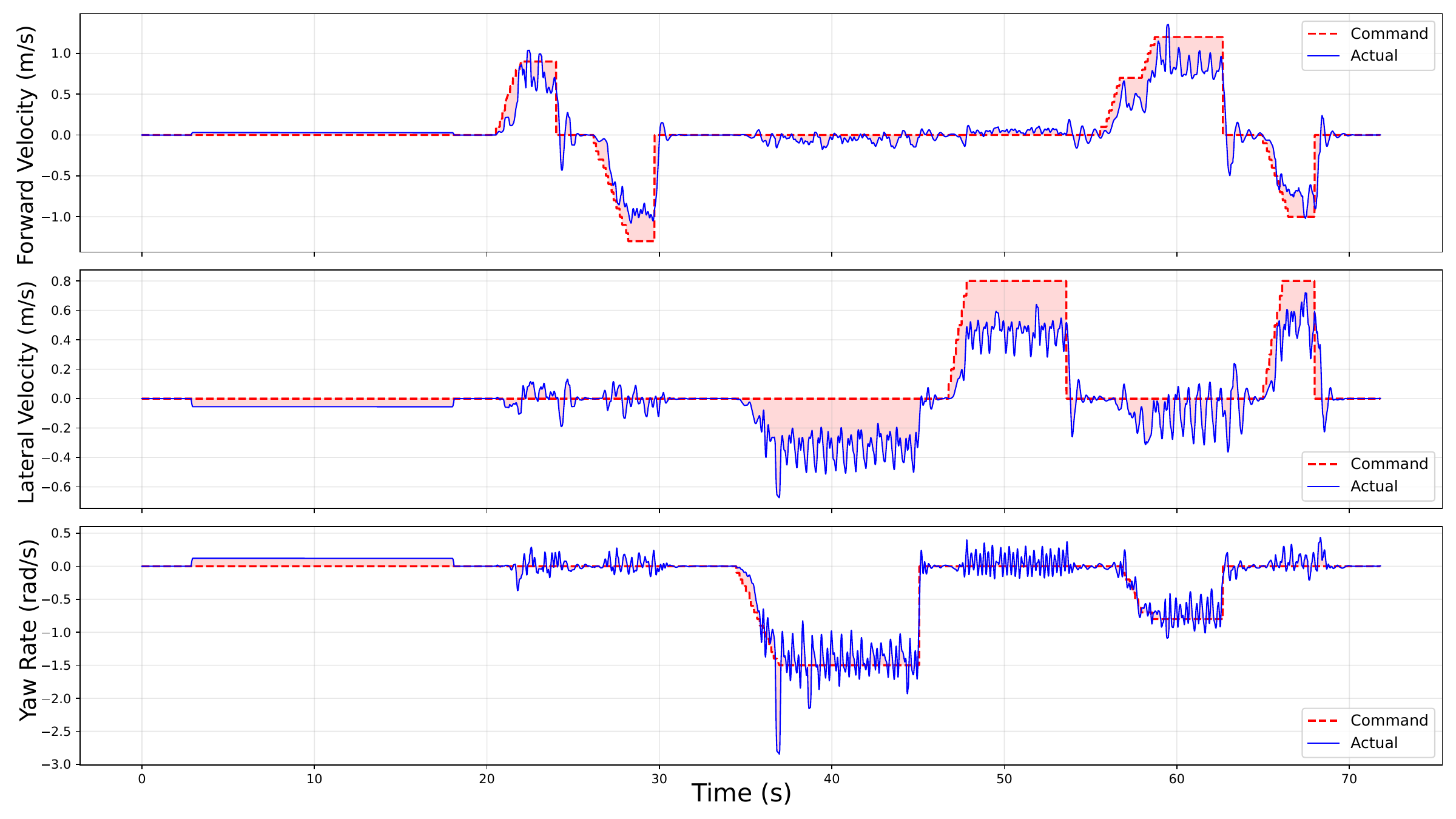}
	\caption{Real-world velocity tracking evaluation.}
	\label{fig:real-eval}
\end{figure}

\autoref{fig:real-eval} shows the velocity-tracking accuracy in the real world. The robot was manually commanded to follow a series of commands involving all three velocity components. The input command was stored as a ROS2 message while a Qualisys camera tracking system recorded the robot's pose. The ground truth velocity was then calculated from the recorded pose using finite differences. While the evaluation shows the tracking is not exact, particularly in the lateral component, the policy follows the commanded velocities with a root-mean-square error below \SI{0.2}{\meter/\second}.
The same policy that achieves speeds above \SI{1}{\meter/\second} during deployment is also capable of withstanding and recovering from strong physical disturbances, as demonstrated by the supplementary videos at \href{https://diffloco.martin-opat.com/}{https://diffloco.martin-opat.com/}.

\section{Discussion and Limitations}
The proposed methodology can discover different gaits, e.g. trotting or galloping, depending on the task requirements, suggesting that it can adapt its locomotion strategy beyond a single fixed behaviour. However, explicit transitions between learned gaits have not yet been investigated and remain a topic for future work. Training for speeds above \SI{2}{\meter/\second} remains difficult, and robust high-speed real-world deployment has not yet been fully achieved. Although FoG methods are approximately $100$x more sample efficient than PPO, as seen in 
\autoref{fig:method-comparison}, their wall-clock training time remains less competitive than that of established frameworks such as RSL-RL~\cite{schwarke2025rslrl}, which provide highly optimised model-free policy-learning algorithms. Additionally, the computational cost per environment step is significantly higher for FoG methods. While model-free algorithms only backpropagate through the neural network, FoG requires analytical backpropagation through the simulator dynamics. Finally, while the methodology can operate with a relatively small number of environments (64--128), scaling to more environments (256--512) remains useful for real-world deployment, since it enables broader domain randomisation.

Compared to SHAC, JAVE improves early training stability by directly supervising the critic Jacobian. This makes it possible to use longer analytical horizon lengths while remaining numerically stable enough to produce deployable policies. However, the improvement is primarily visible in early training and in gradient stability, rather than as a decisive final performance gain. Applying the gradient-supervision idea behind JAVE to more advanced short-horizon methods, such as AHAC~\cite{ahac}, remains a direction for future work. The theoretical conditions under which the Bellman-gradient operator is contractive, as well as practical schedules for restoring the actor dependence in asymmetric actor-critic observations, also require further investigation. See Appendix~\ref{app:jave} for further details.

\section{Conclusion}\label{sec:conclusion}
This work presents Open-DiffLoco, an open-source framework for training deployable blind quadruped locomotion policies with differentiable simulation in MJX. SHAC has been shown to be capable of training robust policies with a simplified differentiable reward set, without kinematic reference trajectories and without privileged base linear velocity in the actor observations. The resulting policies transfer to hardware, track omnidirectional velocity commands, and remain robust to terrain variation and external disturbances. We also introduced JAVE, a SHAC extension that improves early training stability by supervising critic Jacobians. Together, these results show that differentiable locomotion learning can be made reproducible, deployable, and accessible within a mainstream open simulation stack.

\section*{Acknowledgments}
I would like to express my gratitude to the members of the Intelligent Systems Group for insightful conversations during the development of this framework and for conducting the exploratory PPO experiments, as well as to the University of Groningen, the Bernoulli Institute and the Engineering and Technology Institute Groningen for providing the space and equipment necessary for the real-world experiments.

\bibliographystyle{IEEEtran.bst}
\bibliography{bibliography.bib}

\appendices
\section{Jacobian-Augmented Value Estimation (JAVE)}\label{app:jave}
This appendix provides supplementary information to the derivation of the JAVE objective introduced in Section~\ref{subsec: JAVE}.

Recall that $\mathbf o^c$ and $\mathbf o^a$ denote the critic and actor observations, respectively; $\pi$ is the policy; and $\hat{\mathcal G}_\beta$ is the estimated one-step observation map.
The frozen-action construction introduces bias into the target. In the asymmetric-observation setting, freezing the sampled action $\bar{\mathbf a}_t = \mathrm{sg}\{\mathbf a_t\}$ removes the actor-coupling contribution to the closed-loop Jacobian, mediated by the policy derivative $\partial\pi/\partial\mathbf o$. This leaves only the frozen-action Jacobian $\mathbf J_{\mathrm{free}}=\partial\hat{\mathcal G}_\beta/\partial\mathbf o^c$ in the target construction in~\eqref{eq:app-gb-target-wm}. A differentiable ``bridge'' from critic observations $\mathbf o^c$ to actor-compatible observations $\mathbf o^a$ can reduce this bias. Expressing the actor observation as a differentiable function of $\mathbf o^c$ restores a differentiable dependence of the action on the critic observation and hence recovers the feedback term:
\begin{align}
	\tilde{\mathbf o}^a_\eta
	 & = (1-\rho(\eta))\,\mathrm{sg}\{\mathbf o^a\}
	+ \rho(\eta)\,P(\mathbf o^c),                   \\
	\rho(\eta)
	 & = 3\eta^2 - 2\eta^3,\qquad \eta\in[0,1],
	\label{eq:app-smooth-obs-transition}
\end{align}
where $P:\mathcal O_c\to\mathcal O_a$ is a differentiable reconstruction of the actor-compatible observation components. For actor-observation components that are functions of $\mathbf o^c$, this reconstruction is exact, and the corresponding bias is removed. Components of the actor observation $\mathbf o^a$ that are not functions of critic observation $\mathbf o^c$, such as injected observation noise or proprioception history, remain frozen. The smoothstep factor $\rho$ is chosen to have zero slope at both endpoints, interpolating from a fully frozen actor observation ($\eta=0$) to a critic-derived actor observation ($\eta=1$) without introducing a discontinuity. The simplest non-trivial polynomial $\rho$ satisfying these requirements is shown above. With this formulation, the frozen action $\bar{\mathbf a}_t$ is replaced throughout the target construction by $\pi(\tilde{\mathbf o}^a_\eta)$, including in~\eqref{eq:app-wm-closed-loop} and the reward term that enters~\eqref{eq:app-gb-target-wm}. This restores the target's dependence on $\mathbf o^c$. The current implementation corresponds to the fully frozen endpoint ($\eta=0$). Implementing the idea above by experimenting with different schedules for $\eta$ annealing, or by simply setting $\eta=1$, remains future work.

The learned dynamics model is the remaining source of Jacobian error. Fitting $\mathcal F_\psi$ in $L^2$ does not by itself control its Jacobians. Instead, it shifts the gradient-supervision problem from $V_\phi$ to the one-step model. The useful distinction is that the model target is local and can plausibly have bounded curvature. Let $\varepsilon=\hat{\mathcal G}_\beta-\mathcal G_\beta$ on a bounded domain $\Omega\subset\mathbb R^d$ with Lipschitz boundary, and assume $\varepsilon\in H^2(\Omega)$. The Gagliardo--Nirenberg interpolation inequality then gives
\begin{equation}
	\|\nabla\varepsilon\|_{L^2}
	\le
	C\,\|\varepsilon\|_{L^2}^{1/2}\,\|\varepsilon\|_{H^2}^{1/2},
	\qquad C=C(\Omega,d).
	\label{eq:app-gn-ldm}
\end{equation}
The Jacobian error is therefore controlled by the prediction error only in conjunction with the curvature factor $\|\varepsilon\|_{H^2}$, implying that the $L^2$ term alone is insufficient. The bound cannot be strengthened to $\|\nabla\varepsilon\|_{L^2}\le C\|\varepsilon\|_{L^2}$. To see this, the sequence of Proposition~\ref{prop:l2-h1-gap} drives $\|\varepsilon\|_{L^2}\to 0$ while $\|\nabla\varepsilon\|_{L^2}\to\infty$, and so it must be the curvature $\|\varepsilon\|_{H^2}^{1/2}$ that diverges along it. The required curvature bound is thus an assumption on the model and the observation-filtered dynamics, not a consequence of MSE training alone. However, unlike the value-function gradient, this lower-level gap can be closed by supervising the LDM with simulator Jacobians.

In the current implementation, the gradient supervision approach described in this section has been applied to SHAC, which uses a constant analytical horizon. However, the same principle could be applied to adaptive-horizon methods, such as AHAC~\cite{ahac}, which have been shown to outperform SHAC.
Training with the JAVE gradient target converges in practice. Relative to SHAC, JAVE shows improved early convergence across all objectives and increased gradient stability on tasks using the residual formulation of~\cite{bpttMujocoLuo2024}. This permits JAVE to use longer analytical gradient horizons. However, establishing the conditions under which the gradient Bellman operator is guaranteed to contract is left to future work.

Figure~\ref{fig:train_eval} compares SHAC and JAVE. The actor gradients in JAVE are more stable than those in SHAC, resulting in better overall convergence and lower tracking errors. This performance gain is especially visible in early training.
\begin{figure}[tb]
	\centering
	\begin{subfigure}{\linewidth}
		\includegraphics[width=\linewidth]{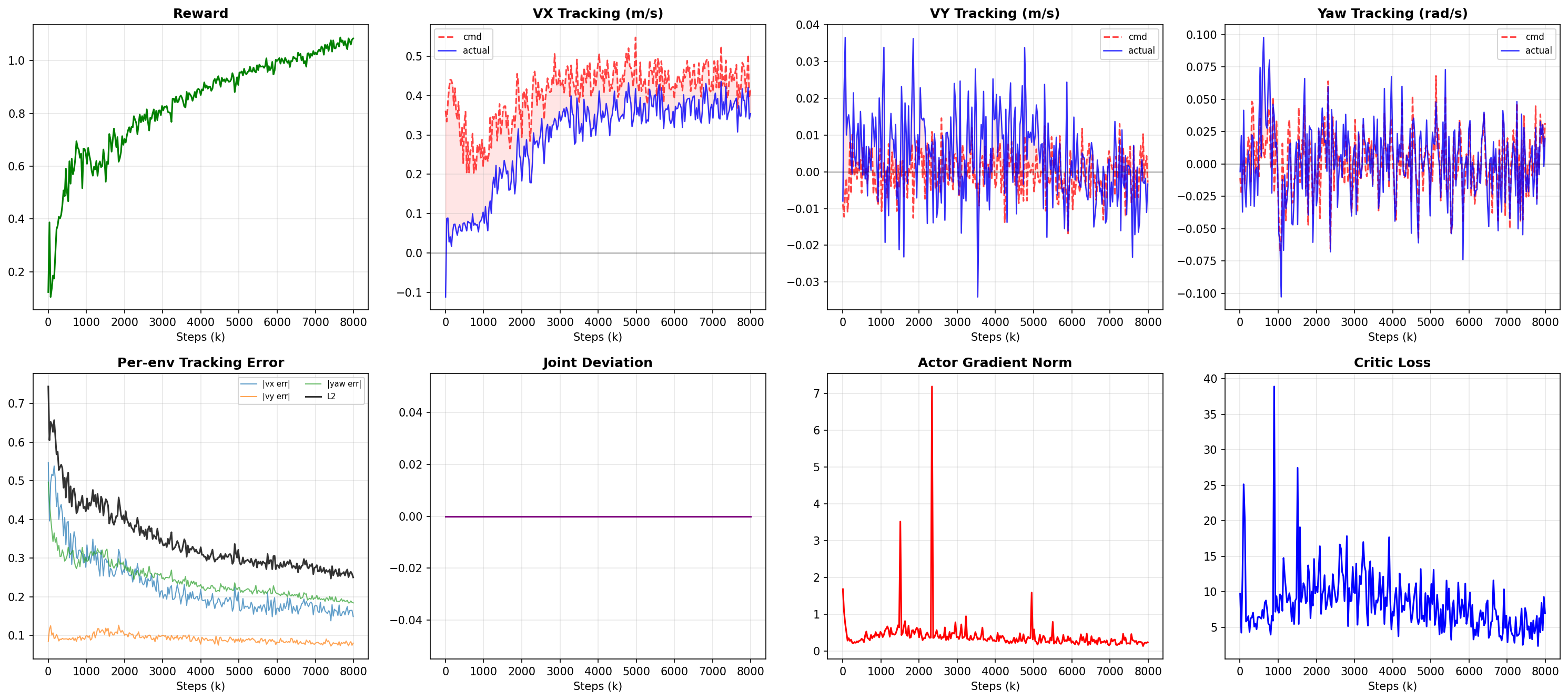}
		\caption{SHAC}
		\label{fig:train_shac}
	\end{subfigure}

	\begin{subfigure}{\linewidth}
		\includegraphics[width=\linewidth]{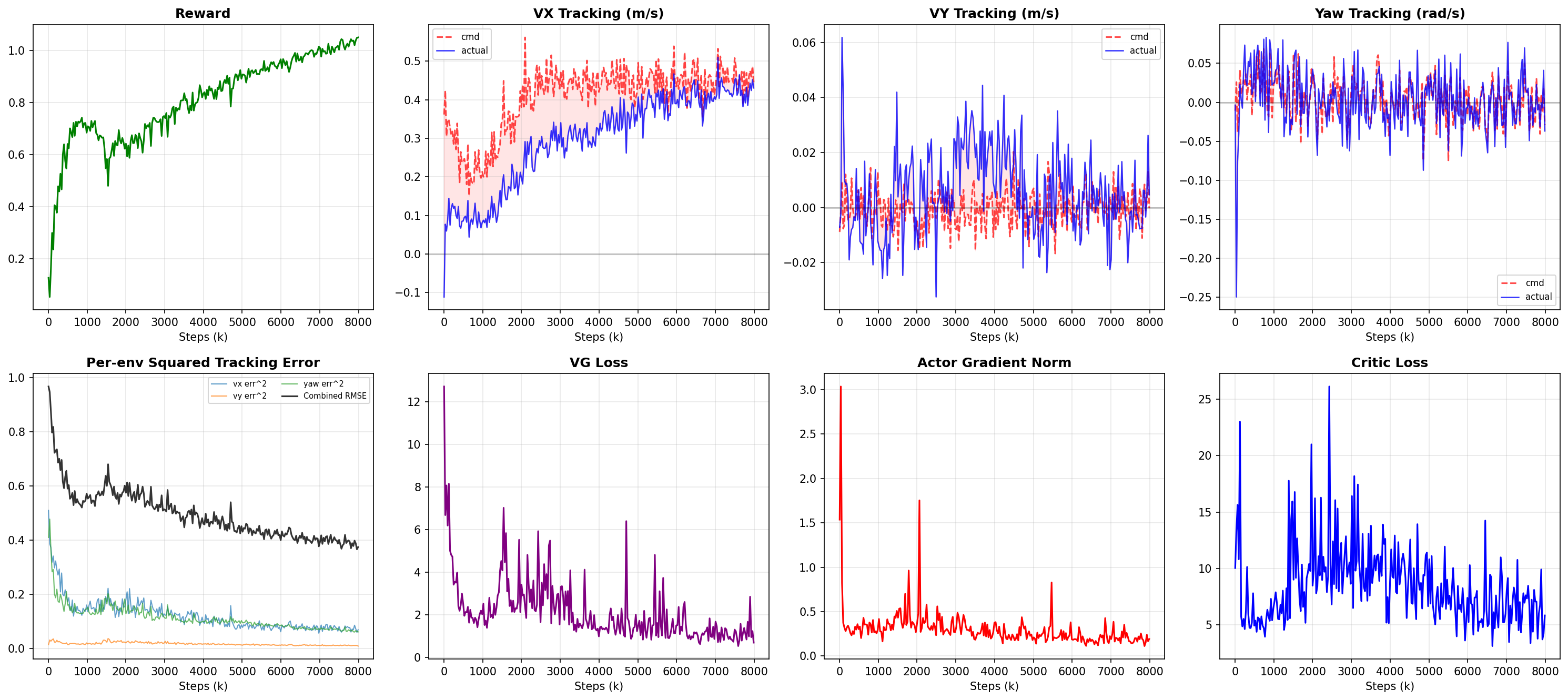}
		\caption{JAVE}
		\label{fig:jave}
	\end{subfigure}

	\caption{Development of several evaluation metrics during SHAC and JAVE training.}
	\label{fig:train_eval}
\end{figure}

\section{Proof of Unboundedness of Gradient Error Under \texorpdfstring{$L^2$}{L2}-bounded Value Error} \label{appendix:grad-unbound-proof}

\begin{proposition}
	\label{prop:l2-h1-gap}
	Let $\Omega \subset \mathbb R^d$ be a bounded set with non-empty interior, and let $\|\cdot\|$ be the $L^2$ norm. There exists a sequence $\{\varepsilon_n\} \subset C^\infty(\Omega),~ n \in \mathbb N$ such that
	\begin{equation}
		\lim_{n \to \infty} \|\varepsilon_n\| = 0 \qquad \text{and} \qquad \lim_{n \to \infty} \|\nabla \varepsilon_n\| = \infty\,.
	\end{equation}
	Consequently, no constant $C < \infty$ satisfies $\|\nabla \varepsilon\| \leq C\,\|\varepsilon\|$ for all $\varepsilon \in C^\infty(\Omega)$.
\end{proposition}

\begin{proof}
	Suppose, for contradiction, that some $C < \infty$ satisfied the inequality $\|\nabla \varepsilon\| \leq C\|\varepsilon\|$ universally.
	It suffices to construct a counterexample sequence on $\Omega_0 = [0, 2\pi] \subset \mathbb R$, as the multidimensional case follows trivially.
	Define $\varepsilon_n: [0, 2\pi] \to \mathbb R$ by
	\begin{equation}
		\varepsilon_n(s) = \frac{1}{n}\sin(n^2 s)\,.
	\end{equation}
	Direct computation gives
	\begin{equation}
		\|\varepsilon_n\|^2 = \int_0^{2\pi} \frac{1}{n^2}\sin^2(n^2 s)\,\mathrm ds =  \frac{\pi}{n^2}, \qquad \|\varepsilon_n\| = \frac{\sqrt{\pi}}{n}\,.
	\end{equation}
	Thus, $\lim_{n\to\infty} \|\varepsilon_n\| = 0$.
	Differentiating
	\begin{gather*}
		\varepsilon_n'(s) = n\cos(n^2 s) \\
		\|\varepsilon_n'\|^2 = \int_0^{2\pi} n^2\cos^2(n^2 s)\,\mathrm ds = \pi n^2,\quad \|\varepsilon_n'\| =\sqrt{\pi}n\,.
	\end{gather*}
	Therefore, $\lim_{n\to\infty} \|\varepsilon_n'\|=\infty$.
	The chosen $\varepsilon_n$ therefore satisfies $\lim_{n \to \infty} \|\varepsilon_n\| = 0 \text{ and } \lim_{n \to \infty} \|\nabla \varepsilon_n\| = \infty$.
	Applying the inequality $\|\nabla \varepsilon\| \leq C\|\varepsilon\|$ to $\varepsilon_n$ yields $\sqrt{\pi}\,n \leq C\frac{\sqrt{\pi}}{n}$, i.e.\ $n^2 \leq C$, which fails for $n^2 > C$ and contradicts the inequality holding universally. Hence, no such constant $C$ exists, and the proof is complete.
\end{proof}

\section{Implicit Terrain}\label{app:implicit-terrain}

In this work, terrain randomisation is used implicitly to expose the policy to uneven surfaces without introducing a mesh height field. Mesh-grid terrain is slower than a flat plane, harder to keep differentiable in contact-rich MJX rollouts, and its exact geometry is unnecessary for blind locomotion, since the policy observes the resulting contact and body dynamics rather than the terrain itself.

The current implicit terrain model has two components. First, slopes are represented by a per-episode tilted gravity vector applied to the robot's centre of mass. For slope angle $\theta$ and uniformly sampled azimuth $\psi$,
\begin{equation}
	\mathbf g_{\mathrm{slope}}
	=
	g
	\begin{bmatrix}
		-\sin\theta\cos\psi &
		-\sin\theta\sin\psi &
		-\cos\theta
	\end{bmatrix}^{\top}.
\end{equation}
This produces slope-like accelerations while keeping the contact geometry flat.

Second, local terrain ``bumps'' are modelled as contact-scaled foot forces based on the differentiated Ornstein--Uhlenbeck (OU) process~\cite{OrigOU, StochasticKleodenPlaten1992}. For each foot $i$, a three-dimensional OU state is updated by
\begin{equation}
	\mathbf u_{i,n+1}
	=
	(1-\gamma)\mathbf u_{i,n}
	+
	\sigma\boldsymbol\varepsilon_{i,n},
	\qquad
	\boldsymbol\varepsilon_{i,n}\sim\mathcal N(\mathbf 0,\mathbf I),
	\label{eq:implicit-terrain-ou}
\end{equation}
and the differentiated process can then be defined as
\begin{equation}
	\Delta \mathbf u_{i,n+1}=\mathbf u_{i,n+1}-\mathbf u_{i,n}.
\end{equation}
The vertical component is clipped to be nonnegative, since foot normal forces should not pull a leg into the ground,
\begin{equation}
	\widetilde{\Delta\mathbf u}_{i,n+1}
	=
	\begin{bmatrix}
		\Delta u^x_{i,n+1} &
		\Delta u^y_{i,n+1} &
		\max(\Delta u^z_{i,n+1},0)
	\end{bmatrix}^{\top}\,.
\end{equation}
The final force applied to the foot body $i$ is
\begin{equation}
	\mathbf F_{i,n+1}
	=
	\frac{f_{i,n}}{\max(W,\varepsilon)}
	\widetilde{\Delta\mathbf u}_{i,n+1},
	\label{eq:implicit-terrain-force}
\end{equation}
where $f_{i,n}$ is the normal contact force from the simulator, $W$ is the robot weight, and $\varepsilon$ is a small numerical constant. Scaling by $f_{i,n}$ makes the disturbance active primarily during contact.

A practical illustration of the terrain generation described above at a single time step is available at \href{https://diffloco.martin-opat.com/terrain-vis/}{https://diffloco.martin-opat.com/terrain-vis/}.

\section{Hyperparameters}
\begin{table}[H]
	\centering
	\caption{SHAC and JAVE training hyperparameters. For more details, see \href{https://github.com/MartinOpat/open-diffloco}{https://github.com/MartinOpat/open-diffloco}.}
	\begin{tabular}{|l|l|}
		\hline
		\textbf{Parameter}           & \textbf{Value}              \\
		\hline
		\multicolumn{2}{|l|}{\textit{General}}                     \\
		\hline
		Total environment steps      & 8\,000\,000 -- 16\,000\,000 \\
		Parallel environments $N$    & 64 -- 256                   \\
		Horizon length $h$           & 16 -- 32                    \\
		Discount $\gamma$            & 0.99                        \\
		GAE $\lambda$                & 0.95                        \\
		\hline
		\multicolumn{2}{|l|}{\textit{Actor}}                       \\
		\hline
		Learning rate                & $5\times10^{-3}$            \\
		Action noise $\sigma$        & 0.5                         \\
		Action scale                 & 0.5                         \\
		\hline
		\multicolumn{2}{|l|}{\textit{Critic}}                      \\
		\hline
		Learning rate                & $5\times10^{-4}$            \\
		Critic updates per iteration & 16                          \\
		Target network update rate   & 0.01                        \\
		\hline
	\end{tabular}
\end{table}

\end{document}